\documentclass[10pt,journal,compsoc]{IEEEtran}
\usepackage{amsmath,amsfonts}
\usepackage{algorithm}
\usepackage{algorithmic}
\usepackage{array}
\usepackage{color}
\usepackage[caption=false,font=normalsize,labelfont=sf,textfont=sf]{subfig}
\usepackage{textcomp}
\usepackage{stfloats}
\usepackage{url}
\usepackage{verbatim}
\usepackage{graphicx}
\usepackage {pdfpages}
\usepackage{multirow}
\usepackage{hyperref}
\usepackage{bm}
\usepackage{amsthm}
\theoremstyle{definition}
\newtheorem{definition}{Definition}
\newtheorem{theorem}{Theorem}

\usepackage[T1]{fontenc}
\usepackage{orcidlink}
\hypersetup{hidelinks,
	colorlinks=true,
	allcolors=black,
	pdfstartview=Fit,
	breaklinks=true}
\def\BibTeX{{\rm B\kern-.05em{\sc i\kern-.025em b}\kern-.08em
    T\kern-.1667em\lower.7ex\hbox{E}\kern-.125emX}}
\usepackage{balance}
\usepackage{soul}
\usepackage{mathrsfs} 
\usepackage{amssymb}
\begin{document}

\title{Spatiotemporal Observer Design for Predictive Learning of High-Dimensional Data}


\author{Tongyi~Liang\orcidlink{0000-0001-8617-2396}
	and~Han-Xiong~Li\orcidlink{0000-0002-0707-5940},~\IEEEmembership{Fellow,~IEEE}
\IEEEcompsocitemizethanks{\IEEEcompsocthanksitem The Authors are with the Department of Systems Engineering, City University of Hong Kong, Hong Kong,  SAR, China.
		E-mail: tyliang4-c@my.cityu.edu.hk;  mehxli@cityu.edu.hk.
	\IEEEcompsocthanksitem H-X~Li is the corresponding author.}
	
\thanks{Manuscript received XX XX, XXXX; revised XX XX, XXXX.}}


\IEEEtitleabstractindextext{
\begin{abstract}
	Although deep learning-based methods have shown great success in spatiotemporal predictive learning, the framework of those models is designed mainly by intuition. How to make spatiotemporal forecasting with theoretical guarantees is still a challenging issue. In this work, we tackle this problem by applying domain knowledge from the dynamical system to the framework design of deep learning models. An observer theory-guided deep learning architecture, called \textit{Spatiotemporal Observer}, is designed for predictive learning of high dimensional data. The characteristics of the proposed framework are twofold: firstly, it provides the generalization error bound and convergence guarantee for spatiotemporal prediction; secondly, dynamical regularization is introduced to enable the model to learn system dynamics better during training. Further experimental results show that this framework could capture the spatiotemporal dynamics and make accurate predictions in both one-step-ahead and multi-step-ahead forecasting scenarios.
\end{abstract}

}

\maketitle

\IEEEdisplaynontitleabstractindextext
\IEEEpeerreviewmaketitle

\IEEEraisesectionheading{\section{Introduction}\label{sec:introduction}}
\IEEEPARstart{S}{patiotemporal} predictive learning, aiming to forecast the future based on past and current observations, is one of the critical topics in spatial-temporal data mining (STDM) \cite{atluri2018spatio}. It has always played a critical role in decision-making and planning in various practices, including climate science, neuroscience, environmental science, health care, and social media. For example, spatial-temporal traffic forecasting can guide transport planning and logistics \cite{lana2018road}. However, the complex spatiotemporal dynamics with high dimensionality increase the difficulties for predictive learning. Furthermore, the property of auto-correlation and heterogeneity makes forecasting extremely challenging \cite{jiang2020survey}. 

Extensive studies have been conducted in STDM communities to tackle this problem. Traditional machine learning methods, like k-nearest neighbors (KNN) \cite{davis1991nonparametric}, support vector machine (SVM) \cite{hong2011traffic}, gaussian process regression (GPR) \cite{sarkka2012infinite} are hard to learn the complex spatiotemporal features. Recently, deep learning-based methods have been applied to make spatiotemporal predictions and achieved remarkable performance. 2D Convolutional neural networks (CNNs) were used in DeepST \cite{zhang2016dnn} and ST-ResNet \cite{zhang2017deep} as their basic layers. Furthermore, 3D CNN was introduced in ST-3DNet \cite{guo2019deep}. Besides, recurrent neural networks (RNNs) and their variants were studied in extensive works like ConvLSTM \cite{shi2015convolutional}, MIM \cite{wang2019memory}, MotionRNN \cite{wu2021motionrnn}.

Though current deep learning-based methods show promising performance, these spatiotemporal models are intuitively designed by trial and error, needing more theoretical analysis. They need more physical explanations and theoretical guidance in model design. Theory guaranteed model is of great significance when we deal with critical issues associated with high risks (e.g., healthcare). How to design a deep learning-based model with theoretical convergence guarantees for prediction should be paid more attention \cite{wahlstrom2015pixels,wang2020deep}. Effective mechanisms could lay solid support for the reasoning of the abstract data \cite{karpatne2017theory}.

To address the issues above, we combine the observer design in control theory with deep learning and propose a theory-guided and guaranteed framework, called \textit{Spatiotemporal Observer}, for spatiotemporal predictive learning. Inspired by the Kazantzis-Kravaris-Luenberger (KKL) observer for traditional low-dimensional systems, we design a spatiotemporal observer for nonlinear systems with high dimensions. The proposed spatiotemporal observer has theoretical guarantees, including the convergence for one-step-ahead forecasting and the upper error bound of multi-step-ahead prediction. As a result, this observer provides a theory-guaranteed architecture for modeling spatiotemporal data. 

Specifically, we first extract low-dimensional representations from original data by a spatial encoder. Then, a spatiotemporal observer is introduced to estimate the future states in the latent space. Finally, the predicted future latent representations are reconstructed to observations via a spatial decoder. Because CNNs show outstanding performance with simplicity and efficiency, we instantiate the proposed framework with CNNs, including 2D convolution layers and inception modules \cite{gao2022simvp}.

The main contributions of this paper are concluded as follows.
\begin{itemize}
	\item A spatiotemporal observer is proposed for predictive learning of high-dimensional data. It can make predictions in one-step-ahead and multi-step-ahead (section \ref{sec_theory}).
	\item 
	We introduce dynamic regularization during the training process, which is beneficial to improving model performance (section \ref{sec_neural_config}). 
	The proposed framework has a theoretical guarantee of convergence and upper bounded error (section \ref{sec_generalization}).
	
	\item We instantiate the proposed spatiotemporal observer with CNNs. Extensive experiments were conducted to validate the performance and effectiveness of the proposed framework. (section \ref{sec: Experiments})
\end{itemize}
\section{Related Work}
With increasing attention drawn to this field, many deep learning-based works have been conducted and achieved significant performance in recent years \cite{wang2020deep, oprea2020review}. These spatiotemporal predictive models can be classified into two main lines: recursive and feedforward.

One line is recurrent models, which employ the RNN-based architecture for future prediction. Specifically, RNNs have been extensively applied in time series modeling \cite{lipton2015critical}. Shi et al. \cite{shi2017deep} integrated CNNs into RNNs architecture for precipitation nowcasting. The proposed convolutional LSTM became a baseline in spatiotemporal predictive learning. After that, many models were presented, for example, PredRNN \cite{wang2017predrnn}, PredRNN++ \cite{wang2018predrnn++}, TrajGRU \cite{shi2017deep}, MSPN \cite{ling2022predictive}, MotionRNN \cite{wu2021motionrnn}, and MS-RNN \cite{ma2022ms}. Recursive models have an advantage in predicting the future with flexible time length. However, recursive models suffer high computational expense and parallelization difficulty because of the chain mechanism of RNNs. To mitigate this problem, researchers proposed a CNN-RNN-CNN framework, using RNNs in the encoded latent space \cite{villegas2017decomposing}. Such methods, like E3D-LSTM \cite{wang2019eidetic}, CrevNet \cite{yu2019efficient}, and PhyDNet \cite{guen2020disentangling}, use CNNs to reduce the spatial size and capture spatial relationships and use RNN to model the temporal dynamics for future prediction.

Another line of research is feedforward models. This kind of method usually stacks CNNs as its backbone due to CNNs' extraordinary success in various tasks in computer vision \cite{voulodimos2018deep}. Oh et al. \cite{oh2015action} proposed a CNNs-based architecture for next-frame prediction in Atari games. Tran et al. \cite{tran2015learning} found that 3D CNNs outperformed 2D CNNs in spatiotemporal learning. To make a more accurate prediction, various sophisticated architectures and strategies were introduced, such as SimVP \cite{gao2022simvp}, DVF \cite{liu2017video}, and  PredCNN \cite{xu2018predcnn}. Thanks to the local connectivity and weights-sharing mechanism, CNNs-based feedforward models typically require fewer computational resources than recurse models. 

However, a common challenging issue exists for both recurse and feedforward models. They need more theoretical designs and explanations for their proposed models. This paper explores a new view of designing network architectures by exploiting observer theory.

\section{Preliminaries and problem statement}\label{sec: Problem Description}
$\mathit{Notation}$: Throughout this work, the general notations are  listed in Table \ref{tab:notation}.

\begin{table}[]
	\centering
	\caption{Table of notation}
	\label{tab:notation}
	\resizebox{!}{!}{%
		\begin{tabular}{c|l}
			\hline
			Notation & Description \\ \hline
	$\mathbb{R}^n$	&   $n$-dimensional Euclidean space    \\ \hline
	$\circ$		&   element-wise Hadamard product          \\ \hline
	$\mathcal{F}_\theta(\cdot)$		&  predictive model with parameters $\theta$         \\ \hline
	$y_k\in\mathbb{R}^{C\times H\times W}$		&   high dimensional observation at time $k$          \\ \hline
	$x_{k}\in\mathbb{R}^{c\times h\times w}$ &	latent representations of $y_k$	\\ \hline
	$z_k\in\mathbb{R}^{c'\times h'\times w'}$	&	latent state at time $k$		\\ \hline
	$\xi_{k}\in\mathbb{R}^{c^*\times h^*\times w^*}$		&  	linear dynamical state at time $k$	\\ \hline
	$f(\cdot)$		&  	transition function	\\ \hline
	$h(\cdot)$		&  	output function 	\\ \hline
	$T(\cdot)$		&  	dynamical transformation function	\\ \hline
	$T^{-1}(\cdot)$		&  	pseudo-inverse function of $T$	\\ \hline
	$\phi_{\theta_1}(\cdot)$	&	spatial encoder		\\ \hline
	$\phi_{\theta_2}^-(\cdot)$  & spatial decoder  \\ \hline
	$A' \in\mathbb{R}^{m\times m}$   & system matrix in KKL  \\ \hline
	$B' \in\mathbb{R}^{m\times p}$   & input matrix in KKL  \\ \hline
	$A\in\mathbb{R}^{c^*\times h^*\times w^*}$   & system coefficient in spatiotemporal observer  \\ \hline
	$B(\cdot)$   &  linear projection in spatiotemporal observer\\ \hline
	$\left| \cdot  \right|$ or $\|\cdot\|_2$ & Euclidean norm \\ \hline
	$\|\cdot\|_\sigma$ & spectral norm	\\ \hline
	$\|\cdot\|_F$ &	Frobenius norm \\ \hline
		 
		\end{tabular}%
	}
\end{table}

\subsection{Spatiotemporal Predictive Learning}
Suppose there is a dynamic system. We take $C$ measurements on a $H\times W$ spatial grid every time step. A tensor $y\in\mathbb{R}^{C\times H\times W}$ can describe observation each time. After that, the history observations for a certain time length $T$ can be expressed as $y_{1:T}=\{y_1, ..., y_T\}$. The future observations with time length $\tau$ is noted as $y_{T+1:T+\tau}=\{y_{T+1},..., y_{T+\tau} \}$, which have a shape of ($\tau$, C, H, W). As shown in Fig. \ref{st_problem}, spatiotemporal predictive learning aims to predict the future observations
$y_{T+1:T+\tau}$ based on the history sequence $y_{1:T}$.
\begin{equation}
	\label{problem}
	y_{T+1},..., y_{T+\tau} = \mathcal{F}_\theta(y_1,...,y_T)
\end{equation}
where $\mathcal{F}_\theta(\cdot)$ is the target predictive model with learnable parameters $\theta$.
\begin{figure}[!t]
	\centering
	\includegraphics[width=2.5in]{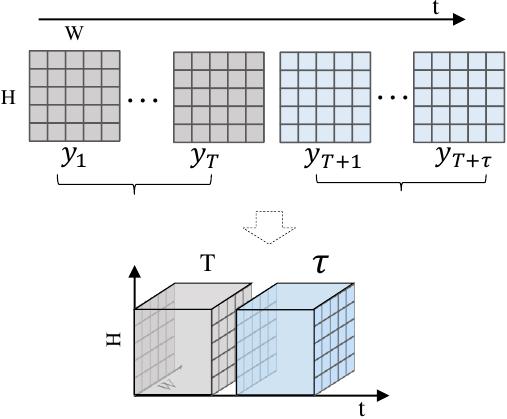}
	\caption{Spatiotemporal nature of forecasting high-dimensional data.}
	\label{st_problem}
\end{figure}
\subsection{Kazantzis-Kravaris-Luenberger Observers}
 Consider a discrete-time system, which has the following general form:
 \begin{equation}
 	\label{nonlinear_system}
 	\begin{cases}
 		z'_{k+1} = f(z'_k) \\
 		x'_k=h(z'_k)
 	\end{cases}
 \end{equation}
with state $z' \in \mathbb{R}^n$, output $x' \in \mathbb{R}^p$, transition function $f(\cdot)$ and output function $h(\cdot)$.	

Typically, we note that system \eqref{nonlinear_system} has nonlinear dynamics because $f(\cdot)$ and $h(\cdot)$ are both nonlinear functions. Function $f(\cdot)$ usually has high complexity, and it is challenging to approximate this function for prediction directly. Therefore, nonlinear reasoning about future states is difficult.

Following \cite{brivadis2019luenberger}, we can assume that there exists a transformation map $T$ that enables

\begin{equation} \label{linear dynamic}
	T(f(z))=A'T(z)+B'h(z) 
\end{equation}
where $A' \in\mathbb{R}^{m\times m}$, $\|A'\|_\sigma<1$, and $B' \in\mathbb{R}^{m\times p}$. Then, the Kazantzis-Kravaris-Luenberger (KKL) observer is such a system:
\begin{equation} \label{observer}
	\begin{cases}
		\xi'_{k} = A' \xi'_{k-1} + B'x'_{k-1}  \\
		\hat{z'}_{k} = T'^{-1}(\xi'_{k})
	\end{cases}
\end{equation}
with $\xi'=T(z') \in \mathbb{R}^{m}, m=(n+1)p$. $T^{-1}$ is a pseudo-inverse of $T$.

We call $T$ a dynamic transformation because $T$ allows us to predict the future by linear inference instead of nonlinear. Therefore, the observer \eqref{observer} provides us an alternate way to make prediction linearly as long as the following equations hold:
\begin{equation}
	\lim_{k \to +\infty} |T(z'_k)-\xi_k(z'_0)| = 0
\end{equation}
\begin{equation}
	\lim_{k \to +\infty} |z'_k - T^{-1}(\xi_k(z'_0))| = 0
\end{equation}

Though KKL has been successfully applied in autonomous and nonautonomous systems with known states \cite{peralez2021deep}, some issues are still unexplored. Firstly, the dynamic transformation $T$ usually exists, if we know the system equation, but it is hard to obtain it in analytical form as examples shown in \cite{brivadis2019luenberger, ramos2020numerical}. Secondly, it is impossible to infer analytical expressions of KKL for systems, if we do not known the system's governing functions. How to design KKL for unknown systems in high dimensions has yet to be designed.

\subsection{Problem Statement}
The problem investigated in this paper can be summarized below:
\begin{itemize}
	\item How to design spatiotemporal observers for modeling and forecasting high-dimensional data with convergence guarantees?
\end{itemize}

\section{Spatiotemporal Observer Design for Learning}
\label{sec: Methodology}
To tackle the abovementioned problem, we propose the \textit{Spatiotemporal Observer} for predictive learning (see Fig. \ref{fig: overall framework}).
\begin{figure*}[!t]
	\centering
	\includegraphics[width=7in]{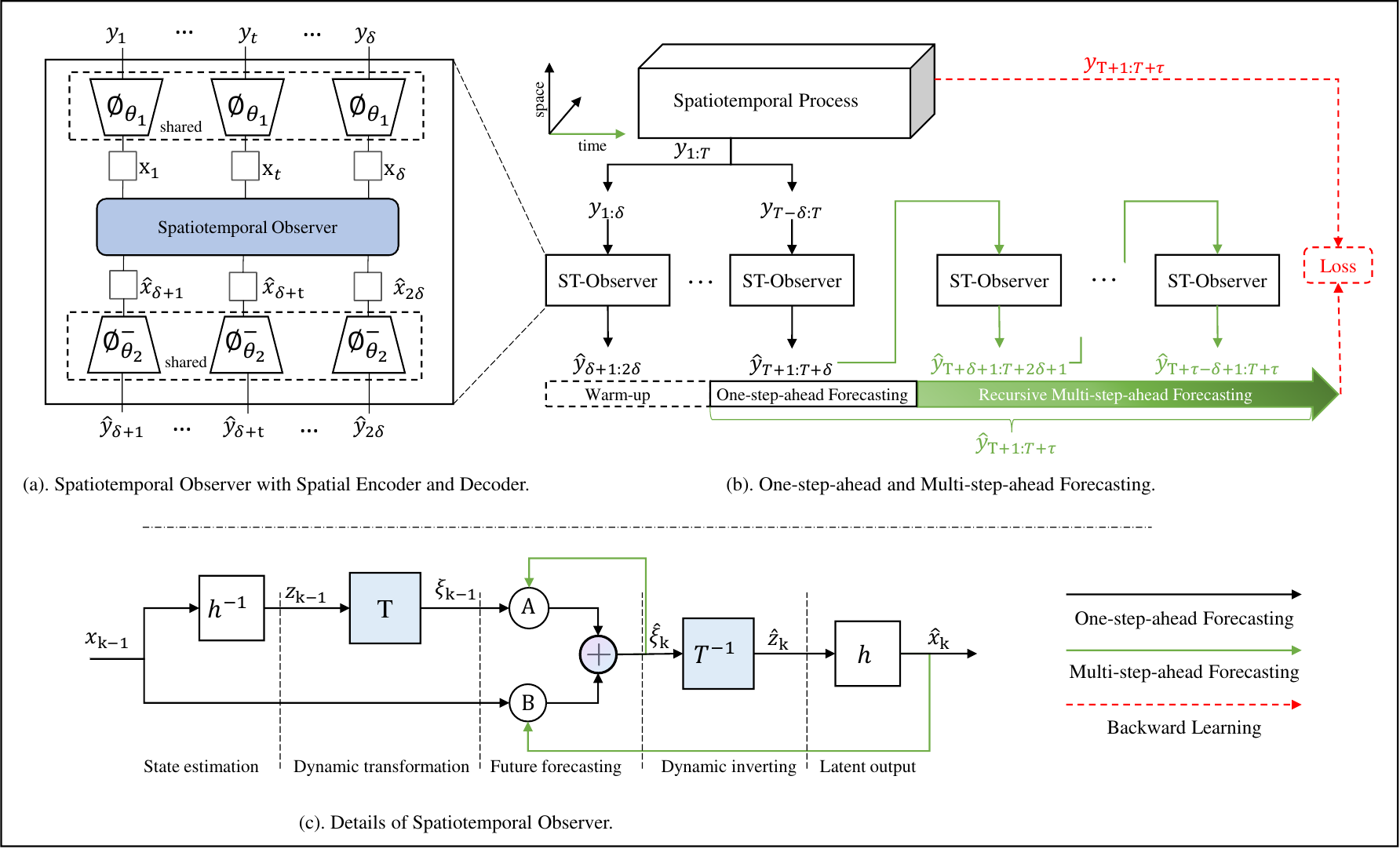}
	\caption{Conceptual framework of spatiotemporal observer for forecasting.}\label{fig: overall framework}
\end{figure*}
\subsection{Theoretical Design} \label{sec_theory}
We first introduce the definition of the Spatiotemporal Observer and then show how to use it for one-step-ahead and multi-step-ahead forecasting.
\subsubsection{Definition of Spatiotemporal Observer}
A spatiotemporal process usually has complex dynamics and high dimensions. To remove the redundant information, we first assume that there exist low-dimensional representations which can capture the dynamics of the original process in the latent space. The latent representations $x_{k}$ are obtained by a spatial encoder $\phi_{\theta_1}$, which decomposes the shared spatial features and retains the dynamics of the observations $y_k$ at time step $k$.
\begin{equation}
	\label{spatial_encoder}
	x_{k} = \phi_{\theta_1}(y_k)
\end{equation}
where $y_k \in\mathbb{R}^{C\times H\times W}$, $x_k \in\mathbb{R}^{c\times h\times w}$, $\theta_1$ is the learnable parameter.

Then, the dynamics of the latent representations can be described by a general state space model \eqref{st__system}. Therefore, the predictive learning problem is equivalent to modeling such a high-dimensional system in the latent space.
\begin{equation}\label{st__system}
	\begin{cases}
		z_{k} = f(z_{k-1}) \\
		x_k=h(z_k)
	\end{cases}
\end{equation}
with latent state $z_k \in\mathbb{R}^{c'\times h'\times w'}$, latent output $x_k \in\mathbb{R}^{c\times h\times w}$, transition function $f(\cdot)$, output function $h(\cdot)$. $c$, $h$ and $w$ represent channels, height, and width. We denote $\xi_k, z_k$ as the true value at time step $k$ of system \eqref{st__system} with $\hat{\xi}_k, \hat{z}_k$ as their estimations under initial state $z_0$.

Compared with system \eqref{nonlinear_system}, system \eqref{st__system} has a higher dimension. A naive way to apply the traditional theory is to vectorize the spatial dimensions of spatiotemporal observations. However, this simple transformation is only suitable for cases with small scales. It means that the traditional KKL observer designed for system \eqref{nonlinear_system} with large scales is not applicable anymore. We should design a new observer for this high-dimensional case instead.

The key to generalizing KKL to the spatiotemporal observer is generalizing matrix multiplication $A' \xi'_{k-1}$ into high dimension.

Given $\xi \in\mathbb{R}^{c^*\times h^*\times w^*}$, we first reshape it into a vector $\xi'=[\xi'_1,...\xi'_m]^T \in\mathbb{R}^{m}$, where $m=c^*h^*w^*$. Because $\|A'\|_\sigma<1$, $A'\in\mathbb{R}^{m\times m}$ in Eq. \eqref{observer} can be equivalently represented as a diagonal matrix $A'=diag\{a_1,..,a_m\}$ using eigenvalue decomposition. Thus, we have
\begin{equation}
	A'\xi' = diag\{a_1,..,a_m\}\times [\xi'_1,...\xi'_m]^T = [a_1\xi'_1,...,a_m\xi'_m]^T 
\end{equation}

Then, we reformulate the result of $A'\xi'$ into a tensor $I$ with shape $(c^*, h^*, w^*)$. $I$ can be further expressed as $I = A \circ \xi$, where $A\in\mathbb{R}^{c^*\times h^*\times w^*}$.

Since $\|A'\|_\sigma<1$, the inequality $\max\{\left| a_i\right|\}<1, \forall i\in\{1, 2, ..., m\}$ holds. Thus, every element in the full tensor $A$ meets $A_{ijl} \in (0, 1), \quad \forall i \in \{1,..., c^*\}, \forall j \in \{1,..., h^*\}, \forall l \in \{1,..., w^*\}$.

Following the KKL observer, we define the spatiotemporal observer as follows.

\begin{definition}[Spatiotemporal Observer]
	Assume that there exists $T: \mathbb{R}^{c'\times h'\times w'}\to \mathbb{R}^{c^*\times h^*\times w^*}$ and it has a pseudo inverse $T^{-1}$. The auxiliary discrete-time system given by
	\begin{equation}
		\label{st_observer}
		\begin{cases}
			\xi_{k} = A \circ \xi_{k-1} + B(x_{k-1})  \\
			\hat{z}_{k} = T^{-1}(\xi_{k})
		\end{cases}
	\end{equation}
	where $A\in\mathbb{R}^{c^*\times h^*\times w^*}$, its element $A_{ijl} \in (0, 1), \quad \forall i \in \{1, ..., c^*\}, \forall j \in \{1, ..., h^*\}, \forall l \in \{1,..., w^*\}$, $\xi_k \in\mathbb{R}^{c^*\times h^*\times w^*}$, and linear projection $B: \mathbb{R}^{c\times h\times w}\to \mathbb{R}^{c'\times h'\times w'}$ , is called a spatiotemporal observer for system \eqref{st__system} if and only if for any initial state $z_0$, the solutions of coupled systems \eqref{st__system} and \eqref{st_observer} satisfy
	\begin{equation}
		\label{y_convergence}
		\lim_{k \to +\infty} |z_k-\hat{z}_k| = 0
	\end{equation}
\end{definition}
\subsubsection{One-step-ahead Forecasting}
Fig. \ref{fig: overall framework} (c) shows the detailed structure of the spatiotemporal observer. This module takes the grouped latent representations $x_{k-1}$ as input and predicts the future representations $x_{k}$. Specifically, the spatiotemporal observer makes prediction using the following five steps.

\textbf{State estimation}. Firstly, according to the relationship between $x_k$ and $z_k$ as described as \eqref{st__system}, when given $x_k$ in one-step-ahead prediction, we can infer $z_k$ by using the inverse function of $h^{-1}$
\begin{equation}\label{eq_h1}
	z_{k-1} = h^{-1}(x_{k-1})
\end{equation}
	
\textbf{Dynamic transformation}. The nonlinear state $z_k$ would be transformed into linear state $\xi_k$ by dynamical transformation $T$.
\begin{equation}\label{eq_T}
	\xi_{k-1} = T(z_{k-1})
\end{equation}

\textbf{Future forecasting}. Then, we can predict the future state $\xi_{k+1}$ linearly by the state transition function as described in observer \eqref{st_observer}.
\begin{equation}\label{eq_pre}
	\hat{\xi}_{k} = A \circ \xi_{k-1} + B(x_{k-1})
\end{equation}
	
\textbf{Dynamic inverting}. After that, the nonlinear state $z_k$ is inferred from $\hat{\xi}_{k}$ by the dynamic inverting function $T^{-1}$. 	
\begin{equation}\label{eq_inv}
	\hat{z}_{k} = T^{-1}(\hat{\xi}_{k})
\end{equation}

\textbf{Latent output}. As described in \eqref{st__system}, the predicted state $\hat{z}_{k}$ is then used as input to the latent output function. And, we get the the predicted latent representation $x_k$.
\begin{equation}\label{eq_out}
	\hat{x}_k=h(\hat{z}_k)
\end{equation}

Finally, we shall complete the one-step-ahead forecasting by reconstructing the future prediction using the spatial decoder $\phi_{\theta_2}^-$ with learnable parameter $\theta_2$.
\begin{equation}\label{eq_decoder}
	\hat{y_k} = \phi_{\theta_2}^-(\hat{x}_{k})
\end{equation}

\subsubsection{Recursive Multi-step-ahead Forecasting}
The ability to predict the future with flexible length is an essential requirement for the broad application of a predictive model. For example, we want to predict the future 10 frames using 5 frames as input. There are two general strategies for multi-step-ahead forecasting \cite{petropoulos2022forecasting}. The one is the recursive strategy, which predicts the future one-step-ahead autoregressively based on former predictions. This strategy enables the predictive model to have efficient parameters and a flexible forecasting horizon.
Nevertheless, the model would be asymptotically biased \cite{terasvirta2010modelling}. The other is the direct strategy, which predicts the multi-step results using once-feedforward computation. It can avoid the accumulation of forecasting errors but lack flexibility.

We utilize a hybrid strategy to trade off the characteristics of recursive and direct methods. The original sequence is divided into several groups. For each forecasting step, the model makes a direct prediction for a short horizon, then autoregressively outputs the final long horizon. For example, a spatiotemporal sequence $\textbf{y}$ contains $T$ number of observations with a $(C, H, W)$ shape. We assign $\delta$ observations into a group. Then, we obtain a new sequence $\textbf{Y}$ with a shape of $(T/\delta,C\delta, H, W)$, as shown in Eq. \eqref{grouping}.

\begin{equation}\label{grouping}
	\textbf{y} : \{y_1, ...,\underbrace{y_{i\delta},..., y_{i\delta+\delta-1}}_{Y_i \, with \,\delta\, elements},...,y_T\}
	\Leftrightarrow 
	\textbf{Y} : \{Y_1, ..., Y_{T/\delta}\}
\end{equation}
where the right arrow denotes the grouping operation, and the left arrow is the reversal degrouping operation. Then, based on the hybrid strategy, we make multi-step ahead prediction by repeatedly applying the spatiotemporal observer \eqref{st_observer}. For instance, the latent linear state can be computed for future $d$ steps as follows.
\begin{equation}
	\begin{cases}
		\hat{\xi}_{k} = A \circ \xi_{k-1} + B(x_{k-1})\\
		\quad\quad\quad\vdots \\
		\hat{\xi}_{k+d} = A \circ \hat{\xi}_{k+d-1} + B(\hat{x}_{k+d-1})\\
	\end{cases}
\end{equation}
A summary of the proposed spatiotemporal observer for multi-step-ahead forecasting is represented in Algorithm \ref{alg_stobserver}.
\begin{figure}[!t]
	\renewcommand{\algorithmicrequire}{\textbf{Input:}}
	\renewcommand{\algorithmicensure}{\textbf{Output:}}
	\begin{algorithm}[H]
		\caption{Spatiotemporal Observer for Forecasting}
		\begin{algorithmic}[1] \label{alg_stobserver}
			\REQUIRE Observations $y_{1:T}=\{y_1, ..., y_t\} \in \mathbb{R}^{T\times C\times H\times W}$         
			\ENSURE Predictions $\hat{y}_{1:\tau} = \{\hat{y}_{1},..., \hat{y}_{\tau}\} \in \mathbb{R}^{\tau\times C\times H\times W}$  
			\STATE $ Y_{1:n} \leftarrow$ group $y_{1:t+\tau}$ into $n$ groups using Eq. \eqref{grouping}  
			\FOR{$k \gets 1$ to $n$} 
			\STATE $x_{k-1}\leftarrow y_{k-1}$ spatial encode using Eq. \eqref{spatial_encoder}.
			\STATE $z_{k-1}\leftarrow x_{k-1}$ state estimation using Eq.\eqref{eq_h1}.
			\STATE $\xi_{k-1} \leftarrow z_{k-1}$  dynamic transformation using Eq. \eqref{eq_T}.
			\STATE $\hat{\xi}_{k} \leftarrow \xi_{k-1}$ future forecasting via Eq. \eqref{eq_pre}.
			\STATE $\hat{z}_{k} \leftarrow \hat{\xi}_{k}$ dynamic inverting via Eq. \eqref{eq_inv}.
			\STATE $\hat{x}_{k} \leftarrow$ $\hat{z}_{k}$ latent output via Eq. \eqref{eq_out}
			\STATE $\hat{Y}_{k} \leftarrow \hat{x}_{k}$ spatial reconstruction using Eq. \eqref{eq_decoder}.
			\ENDFOR  
			\RETURN Predictions.
		\end{algorithmic}
	\end{algorithm}
\end{figure}

\subsection{Neural Configurations for Learning}\label{sec_neural_config}
\subsubsection{Function Approximation via CNNs} \label{function_app}
\begin{figure}[!t]
	\centering
	\includegraphics[width=3.6in]{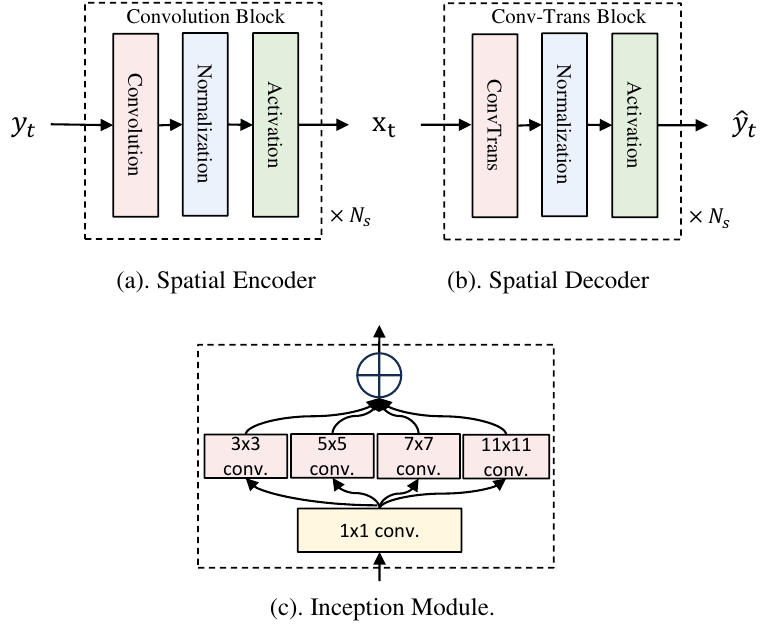}
	\caption{Details of Spatial Encoder, Spatial Decoder and Inception.}
	\label{encoder_deocoder}
\end{figure}
The spatiotemporal observer provides us with a theoretical framework for predictive learning. However, when we use it in pure data cases, the form of the unknown function in the model remains to be determined. As a universal approximator, neural networks, like CNNs, perform excellently in function approximation. In this work, we use CNNs to approximate unknown functions under the proposed framework, including $h$, $T$, and their inverse functions.

We first instantiate the spatial encoder $\phi_{\theta_1}$ in Eq. \eqref{spatial_encoder} with CNNs. Fig. \ref{encoder_deocoder} (a) depicts a schematic diagram of the spatial encoder. It consists of $N_S$ convolution blocks, each containing a convolution 2D layer ($Conv2d$), a normalization layer ($GroupNorm$), and a $LeakyReLU$ activation function ($\sigma$).
\begin{equation}\label{nn_spatioencoder}
	\phi_{\theta_1} (y_k) = [\sigma(GroupNorm(Conv2d))] ^{(N_S)} (y_k)
\end{equation}

For nonlinear functions in the spatiotemporal observer, we approximate them by employing $Inception$ modules, which have succeeded significantly in various vision tasks \cite{szegedy2016rethinking}. As shown in Fig. \ref{encoder_deocoder} (c), Each $Inception$ consists of 4 different convolution filters after a $1\times 1$ convolution. Following \cite{gao2022simvp}, we choose 3, 5, 7, and 11 as kernel sizes of convolution layers in the $Inception$ module. We thus parameterize $T$ and $T^{-1}$ with the $N_T$ $Inception$ modules, and $h$ and $h^{-1}$ with the $N_h$ $Inception$ modules. Taking $T$ as an example, we can express it as follows.
\begin{equation}\label{eq_T_nn}
	T(z_{k-1})  = [Inception] ^{(N_T)} (z_{k-1})
\end{equation}

In future forecasting Eq. \eqref{eq_pre}, we set coefficients $A$ and weights of $B$ as learnable parameters. We utilize functions like sigmoid to achieve $A_{ijl} \in (0, 1)$. $B(x)$ is a linear projection for $x$, and we implement it using a convolution operation. The future forecasting is therefore implemented as Algorithm \ref{alg:generate_a_b} in a Pytorch-like style.
\begin{figure}[!t]
	\renewcommand{\algorithmicrequire}{\textbf{Input:}}
	\renewcommand{\algorithmicensure}{\textbf{Output:}}
	\begin{algorithm}[H]
		\caption{Calculate Eq. \eqref{eq_pre} with Learnable $A$ and $B$}
		\begin{algorithmic}[1] \label{alg:generate_a_b}
			\REQUIRE $\xi_{k-1}$, $x_{k-1}$          
			\ENSURE predicted $\xi_{k}$  
			\STATE {$A$ = torch.nn.Parameter(torch.empty($T_c$, n))}
			\STATE {$B$ = torch.nn.Conv2d($T_S$, $T_c$)} \qquad \qquad \# or Inception
			\STATE {$A$ = torch.nn.functional.sigmoid($A$)} 
			\STATE {$\xi_k$ = $A\circ\xi_{k-1}$ + $B(x_{k-1})$}
			\STATE {Return $\xi_k$}
		\end{algorithmic}
	\end{algorithm}
\end{figure}

The spatial decoder $\phi_{\theta_2}$, as shown in \ref{encoder_deocoder} (b), utilizes the convolution transpose blocks as its backbone. This module consists of a 2D transposed convolution layer ($Conv2dTrans$), a layer normalization layer ($GroupNorm$), and a $LeakyReLU$ activation function ($\sigma$). Like the spatial encoder, the spatial decoder stacks $N_S$ deconvolution modules.
\begin{equation}\label{nn_spatiodecoder}
	\phi_{\theta_2}^-(\hat{x}_{k})= [\sigma(GroupNorm(Conv2dTrans))] ^{(N_S)}(\hat{x}_{k})
\end{equation}

It should be noted that there is an inverse relationship between $h$ and $h^{-1}$, $T$ and $T^{-1}$, and spatial encoder and spatial decoder. We use skip connections between them for better information transformation and reconstruction.
 
\subsubsection{Learning with Dynamical Regularization}
To learn the parameters of the proposed model, we design the overall objective function with dynamical regularization. $\mathcal{L}_y$ has two terms $L_2$ and $L_1$ loss of predicted frames, which enables the model to learn the smoothness and sharpness of frames. 
\begin{equation} 
	\label{loss_y}
	\mathcal{L}_y  = \frac{1}{NL}\sum_{i=1}^{N}\sum_{j=1}^{L}\left\| y_{ij} - \hat{y}_{ij}\right\|_2^2 + \lambda_0\left\| y_{ij} - \hat{y}_{ij}\right\|_1 
\end{equation}
where $N$ is the batch size, and $L$ is length of the future sequence. 

We also force the predicted latent representation $\hat{x}$ to be close to the latent representation $x$ of the ground true of future observation.
\begin{equation} 
	\label{loss_x}
	\mathcal{L}_x  =  \frac{1}{NL}\sum_{i=1}^{N}\sum_{j=1}^{L}\left\| x_{ij} - \hat{x}_{ij}\right\|_2^2	
\end{equation}
Additionally, as described in the spatiotemporal observer, we shall know that the estimated state $z$ and transformed $\xi$ from future observations can serve as the true value of the prediction. So, we can set $\mathcal{L}_z$ and $\mathcal{L}_{\xi}$ as dynamical regularization by minimizing the $L_2$ distance between the predicted value and the ground truth.
\begin{equation} 
	\label{loss_z}
	\mathcal{L}_z  =  \frac{1}{NL}\sum_{i=1}^{N}\sum_{j=1}^{L}\left\| z_{ij} - \hat{z}_{ij}\right\|_2^2	
\end{equation}

\begin{equation} 
	\label{loss_xi}
	\mathcal{L}_{\xi}  =  \frac{1}{NL}\sum_{i=1}^{N}\sum_{j=1}^{L}\left\| \xi_{ij} - \hat{\xi}_{ij}\right\|_2^2	
\end{equation}
Combining everything, we define the overall loss function as follows.
\begin{equation} 
	\label{loss_all}
	\mathcal{L}  = \mathcal{L}_y + \lambda_1\mathcal{L}_x + \lambda_2\mathcal{L}_z + \lambda_3\mathcal{L}_{\xi}
\end{equation}

\subsection{Generalization and Convergence Analysis} \label{sec_generalization}
In this section, we give the theoretical analysis of generalization error and convergence of the proposed spatiotemporal observer.
\subsubsection{Generalization Error Bound}
Suppose the training examples $S=((x_1, y_1), ...,(x_n, y_n))$ drawn i.i.d. according to an unknown distribution $\mathcal{D}$. Let $X=(x_1,...,x_n)$ be the input and $Y=(y_1,...,y_n)$ be the output. Suppose the hypothesis space computed by the spatiotemporal observer is $\mathcal{H}$. We denote the loss function  $\mathscr{L}_\eta$ to measure the prediction error. Assume the loss function  $\mathscr{L}_\eta$ is $\eta$-Lipschitz and is upper bounded by $M>0$. The forecasting problem is using the training samples to find a hypothesis $h\in\mathcal{H}$ with the expected risk or generalization error defined as
\begin{equation}
	\mathcal{R}_D(h)=\underset{(x,y)\sim D}{\mathbb{E}}\left[ \mathscr{L}_\eta(h(x)),y \right]
\end{equation}
The empirical risk is denoted as
\begin{equation}
	\hat{\mathcal{R}}_{S|\mathscr{L}_\eta}(h)=\frac{1}{n}\sum_{i=1}^{n}\left[ \mathscr{L}_\eta(h(x_i)),y_i \right]
\end{equation}

Based on the covering number analysis of the spatiotemporal observer, we obtain the following generalization bound.
\begin{theorem}\label{thm:n_prediction}
	Given $n$ training samples $S=((x_1, y_1), ...,(x_n, y_n))$  drawn i.i.d. according to distribution $\mathcal{D}$. A hypothesis $h\in\mathcal{H}$ is defined by the spatiotemporal observer. Let $X=(x_1,...,x_n)$ be the input and loss function  $\mathscr{L}_\eta$ be $\eta$-Lipschitz and upper bounded by $M>0$. Then with probability at least $1-\delta$, each hypothesis $h$ satisfies
	\begin{equation}\label{eq_genebound}
		\mathcal{R}_D(h) \le \hat{\mathcal{R}}_{S|\mathscr{L}_\eta}(h) + 32n^{-5/8}\left( \frac{\left\| X \right\|_F\mathscr{R}}{\eta} \right)^{1/4}
		+ M\sqrt{\frac{\log\frac{1}{\delta}}{2n}} 
	\end{equation}
	where $\mathscr{R}$ is defined in Eq \eqref{eq_rademancher}.
\end{theorem}

\begin{proof}
		Given a neural network with $L$ layers, we denote the input of the $i_{th}$ layer as $X_i=(x_1^i,..,x_n^i)^T \in\mathbb{R}^{d_i\times n}, i=1,...,T$ with the number of the training samples $n$ and the size of each sample $d_i$. The nonlinear function $\sigma_i$ is assumed to be $\rho_i$-Lipschitz satisfying $\sigma_i(0)=0$.
	
	Following \cite{lin2019generalization}, the convolution operation between a convolutional weight $W_i=(w^1,..,w^c) \in\mathbb{R}^{c\times r}$ and input $X_i$, where $c$ denotes the number of kernels and $r$ denotes the size of kernels, can be formulated as
	\begin{equation} \label{eq_conv}
		\mu_i(W_i,X_i)=\gamma_i(W_i)X_i
	\end{equation}
	where $\gamma_i(W_i)\in\mathbb{R}^{d_i\times d_{i-1}} $ is the matrix generated by convolutional weight $W_i$.
	
	The $Inception$ module used in this work has two parts. The first part is $1\times 1$ convolution, which could be expressed using Eq. \eqref{eq_conv}. The second part can also be written into Eq.\eqref{eq_conv} via the distributivity of convolution, $\sum\mu_i(W,X)=\mu_i(\sum W,X)$, where $\sum W$ is a kernel generated by summation of all kernels. Two consecutive convolutions are equivalent to one convolution. Therefore, the $Inception$ module can be expressed as a single CNN layer.
	
	A neural network with $L$ layers can be formulated as
	\begin{equation}\label{expre_nn}
		F(X):= \sigma_L(\gamma_L(W_L)\sigma_{L-1}(\gamma_{L-1}(W_{L-1})...\sigma_1(\gamma_1(W_1)X_1)...))
	\end{equation}
	where $X_1$ is the first layer input and also the model input $X$.
	
	Then, the functions parameterized by CNNs in Section \ref{function_app} are expressed as $F_{\phi_{\theta_1}}$, $F_{\phi_{\theta_2}^-}$
	$F_{h^{-1}}$, $F_{T}$, $F_{T^{-1}}$, $F_{h}$. The spatiotemporal observer is a special case of the stem-vine framework \cite{he2020resnet}. The stem can be formulated as
	\begin{equation}\label{expre_stem}
		F_S(X):=F_{\phi_{\theta_2}^-}(F_{h}(F_{T^{-1}}(AF_{T}(F_{h^{-1}}(F_{\phi_{\theta_1}}(X_1))))))
	\end{equation}
	where weight $A\in\mathbb{R}^{d_A\times d_A}$ is obtained from $A$ in Eq.\eqref{eq_pre}. Considering operation with $A$ as a single layer, the total number of layers in $F_S(X)$ is $L_S=2*(N_S+N_h+N_T)+1$.
	The vine computes a function with input $X_V=F_{\phi_{\theta_1}}(X_1)$ as
	\begin{equation}\label{expre_vine}
		F_V(X)=\sigma_B(\gamma_B(W_B)X_V)
	\end{equation}
	where $\sigma_B$ and $\gamma_B(W_L)$ are from $B$. If we set $B$ as linear convolution operation without nonlinear activation in Eq.\eqref{eq_pre}, $\sigma_B=1$.
	
	We then give covering number bound for $F_S$ and $F_V$. Denote the hypothesis space computed by $F_S$ and $F_V$ are $\mathcal{H}_S$ and $\mathcal{H}_V$. Let nonlinearity $(\sigma_1,...,\sigma_A,...,\sigma_{L_S},\sigma_B)$ be a fixed function where $\sigma_i$ is assumed to be $\rho_i$-Lipschitz satisfying $\rho_i(0)=0$. Let $(a_1,...,a_A,...,a_{L_S},a_B)$ and $(s_1,...,s_A,...,s_{L_S},s_B)$ be some real values. Assuming $\| Wi \|_F\le a_i$ and $\|\gamma_i(W_i)\|_\sigma\le s_i$, $\|W_B\|\le a_B$ and $\|\gamma_B(W_B)\|_\sigma\le s_B$, and $\| A \|_F\le a_A$, and $\|A\|_\sigma\le s_A$. Then, using Lemma 14 in \cite{lin2019generalization}, we have the following covering number bounds for $F_S$ and $F_V$.
	\begin{equation}
		\ln \mathcal{N}\left( \mathcal{H}_S,\epsilon,\left\| \cdot  \right\|_F \right)\le \left( \frac{\left\| X \right\|_F\mathscr{R}_S}{\epsilon} \right)^{1/2}
	\end{equation}
	with $\mathscr{R}_S$ defined as
	\begin{equation}\label{eq_RS}
		\mathscr{R}_S=\left( 2\prod_{i=1}^{L_S}\rho_is_i \right)\left( \frac{d_A^4a_A}{s_A}+\sum_{i\neq A }^{L_S-1}\frac{c_i^2r_i^2a_i\sqrt{d_i/c_i}}{s_i} \right)L_S^2
	\end{equation}
	
	\begin{equation}
		\begin{split}
			\ln \mathcal{N}\left( \mathcal{H}_V,\epsilon,\left\| \cdot  \right\|_F \right)& \le \left( \frac{\left\| X_V \right\|_F\mathscr{R}_V^{'}}{\epsilon} \right)^{1/2} \\
			& = \left( \frac{\left\| F_{\phi_{\theta_1}}(X_1) \right\|_F\mathscr{R}_V^{'}}{\epsilon} \right)^{1/2} \\
			& \le \left( \frac{\left\| X \right\|_F \prod_{i=1}^{N_S}\rho_is_i \mathscr{R}_V^{'}}{\epsilon} \right)^{1/2} \\
			& = \left( \frac{\left\| X \right\|_F \mathscr{R}_V}{\epsilon} \right)^{1/2} \\
		\end{split}
	\end{equation}
	
	where $\mathscr{R}_V^{'}$ and $\mathscr{R}_V$ are defined as
	\begin{equation}
		\mathscr{R}_V^{'} =2\rho_Bc_B^2r_B^2a_B\sqrt{d_i/c_i}
	\end{equation}
	\begin{equation}\label{eq_RV}
		\mathscr{R}_V=2\prod_{i=1}^{N_S}\rho_is_i \left( \rho_Bc_B^2r_B^2a_B\sqrt{d_B/c_B} \right)
	\end{equation}
	
	Theorem 1 in \cite{he2020resnet} introduces that the covering number of a deep neural network constituted by a stem and a series of vines is upper bounded by the product of the covering numbers of stem and vines. Thus, we can obtain the covering number bound for the
	spatiotemporal observer.
	
	Suppose the hypothesis space computed by the spatiotemporal observer is $\mathcal{H}$. Then we have
	\begin{equation}
		\ln \left( \mathcal{N}\left( \mathcal{H},\epsilon,\left\| \cdot  \right\|_F \right) \right)
		\le \ln\left( \mathcal{N}\left( \mathcal{H}_S,\epsilon,\left\| \cdot  \right\|_F \right)\cdot \mathcal{N}\left( \mathcal{H}_V,\epsilon,\left\| \cdot  \right\|_F \right) \right) \\	
	\end{equation}
	
	Given $n$ training samples $S=((x_1, y_1), ...,(x_n, y_n))$. Assume the loss function  $\mathscr{L}_\eta$ is $\eta$-Lipschitz and is upper bounded by $M>0$. We define $\mathscr{L}_\eta$ with respect to $\mathcal{H}$ as
	\begin{equation}
		\mathcal{H}_\eta:=\left\{ (x,y)\to \mathscr{L}_\eta(h(x),y) : h\in \mathcal{H}\right\}
	\end{equation}
	Since  $\mathscr{L}_\eta$ is $\eta$-Lipschitz, we have
	\begin{equation}
		\begin{split}
			&\ln\mathcal{N}\left( \mathcal{H}_{\eta|S},\epsilon,\left\| \cdot  \right\|_F \right)
			\le \ln\mathcal{N}\left( \mathcal{H}_{|X},\eta\epsilon,\left\| \cdot  \right\|_F \right)\\
			&\le \ln\left( \mathcal{N}\left( \mathcal{H}_S,\eta\epsilon,\left\| \cdot  \right\|_F \right)\cdot \mathcal{N}\left( \mathcal{H}_V,\eta\epsilon,\left\| \cdot  \right\|_F \right) \right) \\
			&=\left( \frac{\left\| X \right\|_F\mathscr{R}_S}{\eta\epsilon} \right)^{1/2}+\left( \frac{\left\| X \right\|_F \mathscr{R}_V}{\eta\epsilon} \right)^{1/2}\\
			&=\left( \frac{\left\| X \right\|_F\mathscr{R}}{\eta\epsilon} \right)^{1/2}\\
		\end{split}
	\end{equation}
	where $\mathscr{R}$ is expressed as
	\begin{equation}\label{eq_rademancher}
		\mathscr{R}=\left( \sqrt{\mathscr{R}_S} + \sqrt{\mathscr{R}_V}\right)^2
	\end{equation}
	
	We then relate the covering bound for the spatiotemporal observer to the empirical Rademacher complexity by Dudley’s entropy integral.
	\begin{equation}
		\begin{split}
			&\mathfrak{R}_S(\mathcal{H}_\eta)
			\le \inf_{\alpha > 0} \left( \frac{4\alpha}{\sqrt{n}} +\frac{12}{n}\int_{\alpha}^{\sqrt{n}}\sqrt{ln\mathcal{N}\left( \mathcal{H}_{\eta|S},\epsilon,\left\| \cdot  \right\|_F \right)}d\epsilon\right)\\
			&\le \inf_{\alpha > 0} \left( \frac{4\alpha}{\sqrt{n}} +\frac{12}{n}\int_{\alpha}^{\sqrt{n}}\left( \frac{\left\| X \right\|_F\mathscr{R}}{\eta\epsilon} \right)^{1/4}d\epsilon\right) \\
			&= \inf_{\alpha > 0} \left( \frac{4\alpha}{\sqrt{n}} +\frac{16}{n}\left( \frac{\left\| X \right\|_F\mathscr{R}}{\eta} \right)^{1/4}\left( n^{3/8}-\alpha^{3/4} \right)\right)\\
		\end{split}
	\end{equation}
	Let the first derivative of the right-hand side equal to zero, we obtain the minimum at $\alpha= \frac{81\left\| X \right\|_F\mathscr{R}}{\eta n^2}$. We further have  the Rademacher complexity
	\begin{equation}\label{eq_erc}
		\begin{split}
			\mathfrak{R}_S(\mathcal{H}_\eta)
			&\le16n^{-5/8}\left( \frac{\left\| X \right\|_F\mathscr{R}}{\eta} \right)^{1/4} -\frac{108\left\| X \right\|_F\mathscr{R}}{\eta n^{5/2}}\\
			&\le 16n^{-5/8}\left( \frac{\left\| X \right\|_F\mathscr{R}}{\eta} \right)^{1/4}\\
		\end{split}
	\end{equation}
	
	The generalization bound for regression is introduced in \cite{mohri2018foundations}. The generalization error $\mathcal{R}_D(h)$ with respect to target $f$ is bounded as follows.
	
	(\textit{Theorem 11.3}, \cite{mohri2018foundations}.)
	Given hypothesis $\mathcal{H}$, training samples $S=((x_1, y_1), ...,(x_n, y_n))$, with probability at least $1-\delta$, each hypothesis $h\in\mathcal{H}$ satisfies
	\begin{equation}\label{ref_eq_genebound}
		\mathcal{R}_D(h) \le \hat{\mathcal{R}}_{S|\mathscr{L}_\eta}(h) + 2\mathfrak{R}_S(\mathcal{H}_\eta)
		+ M\sqrt{\frac{\log\frac{1}{\delta}}{2n}} 
	\end{equation}
	
	Substituting the Rademacher complexity $\mathfrak{R}_S(\mathcal{H}_\eta)$ from Eq. \eqref{eq_erc} into Rademacher complexity regression bounds Eq. \eqref{ref_eq_genebound}, we obtain the generalization bound Eq.\eqref{eq_genebound} for the spatiotemporal observer.
	
	The proof is completed.
\end{proof}

\autoref{thm:n_prediction} implies that every time the spatiotemporal observer makes a single-step prediction, the prediction error is upper bounded. 

\subsubsection{Convergence Analysis}
For long sequences, the model would make single-step predictions continuously over time. The spatiotemporal observer has good convergence properties, such that the predicted values will gradually converge to the ground truth. We, therefore, derive the following theorem.

\begin{theorem}\label{thm:stobserver}
	Coefficients $A\in\mathbb{R}^{c^*\times h^*\times w^*}$ is a full tensor and  its element $A_{ijl} \in (0, 1), \quad \forall i \in \{1, ..., c^*\}, \forall j \in \{1, ..., h^*\}, \forall l \in \{1,..., w^*\}$, $\xi_k \in\mathbb{R}^{c^*\times h^*\times w^*}$. $B: \mathbb{R}^{c\times h\times w}\to \mathbb{R}^{c'\times h'\times w'}$ is a linear function. Let $T: \mathbb{R}^{c'\times h'\times w'}\to \mathbb{R}^{c^*\times h^*\times w^*}$ be a continuous map. Assume that:
	\begin{enumerate}
		\item For any $z_k \in\mathbb{R}^{c'\times h'\times w'}$, $T$ is uniformly injective and satisfies
		\begin{equation}
			T(f(z)) =  A \circ T(z) +  B(x)
		\end{equation}
		\item There exists a function $\alpha $ for any given $(z_1,z_2)$, the following equation holds
		\begin{equation}
			|z_1-z_2| \le  \alpha (|T(z_1)-T(z_2)|)
		\end{equation}
	\end{enumerate}
	Then there exists a function $T^*: \mathbb{R}^{c'\times h'\times w'}\to \mathbb{R}^{c^*\times h^*\times w^*}$ such that $\hat{z}_k = T^*(\xi_k)$ are the solutions of the spatiotemporal observer for system \eqref{st__system}.
\end{theorem}
\begin{proof}
	let $\xi_k(x_0,z_0, \xi_0)$, abbreviated as $\xi_k$, be the solution of equation \eqref{st_observer}. Since every element in the full tensor $A$ meets $A_{ijl} \in (0, 1), \quad \forall i \in \{1,..., c^*\}, \forall j \in \{1,..., h^*\}, \forall l \in \{1,..., w^*\}$, and $T$ ensures that \eqref{st_observer} is satisfied, thus
	\begin{align*}
		& \lim_{k \to +\infty} |\xi_k-\hat{\xi}_k|                                                      \\
		& =\lim_{k \to +\infty } |T(z_k) - \hat{\xi}_k |                                                \\
		& =\lim_{k \to +\infty } |( A\circ T(x_{k-1}) +B(x_{k-1}) ) - (A\circ \hat{\xi}_{k-1} +B(x_{k-1}) )| \\
		& =\lim_{k \to +\infty } |A^{k-1}\circ\left( T(z_0) - \hat{\xi}_0 \right)|                           \\
		& =0
	\end{align*}
	
	Since $T$ is uniformly injective, it means that there exists a class $\mathcal{K}^\infty$ function $\alpha $. Based on the definition of $\mathcal{K}^\infty$ function, $\alpha $ is a nondecreasing positive definite function.
	
	Because of the uniform injectivity of $T$, there exists a pseudo-inverse $T^{-1}: \mathbb{R}^{m\times h\times w}\to \mathbb{R}^{c\times h\times w}$ such that the following equation holds.
	\begin{equation}
		\label{T1}
		|T^{-1}(\xi_1) - T^{-1}(\xi_2)| \le \alpha (|\xi_1-\xi_2|)
	\end{equation}
	Based on Theorem 2 in \cite{mcshane1934extension}, there exist $T^{*}$, which is an extension of $T^{-1}$, satisfying \eqref{T1}. Let $z = T^{*}(\xi)$, $\hat{\xi}_k = T(z)$, and we can get
	\begin{equation}
		|z_k - \hat{z}_k| \le \alpha (|\xi_k-\hat{\xi}_k|)
	\end{equation}
	As $k \to +\infty $, $|z_k - \hat{z}_k|\to 0$. 
	
	Here we complete the proof.
\end{proof}
\autoref{thm:n_prediction} and \autoref{thm:stobserver} provide us with theoretical guarantees, ensuring that the proposed model has an upper bound when making predictions and will gradually converge over time. 

\section{Experiments}
\label{sec: Experiments}
\subsection{Implementation}
To measure the performance and effectiveness of our proposed framework, we evaluate the spatiotemporal observer for one-step-ahead and multi-step-ahead forecasting cases using a real-world traffic flow dataset (TaxiBJ \cite{zhang2017deep}), a synthetic dataset (Moving MNIST) and a radar echo dataset (the CIKM
AnalytiCup 2017 competition dataset\footnote{\href{https://tianchi.aliyun.com/competition/entrance /231596/information}{https://tianchi.aliyun.com/competition/entrance/231596/inform
		ation}}, abbreviated as CIKM). The source code and trained models are available online 
\href{https://github.com/leonty1/Spatiotemporal-Observer}{https://github.com/leonty1/Spatiotemporal-Observer}.

Table \ref{tab:exp_setup} lists the main hyperparameters used in each experiment. Each dataset has $N_{train}$, $N_{val}$, and $N_{test}$ samples in the training, validation, and testing process. The time length of input and output are $T_{in}$ and $T_{out}$. The spatial encoder has $N_S$ convolution blocks with $C_S$ channels. The spatial decoder has $N_S$ convolution transpose blocks with $C_S$ channels. The dynamic transition function $T$ and its inverse $T^{-1}$ have $N_T$ inception blocks with $C_T$ channels. $h$ and $h^{-1}$ have $N_h$ inception blocks with $C_h$ channels. We use Adam \cite{kingma2014adam} as the optimizer and train the model using loss \eqref{loss_all} with coefficients listed in Table \ref{tab:exp_setup}. Batch size, learning rate (LR), and epochs are also given in Table \ref{tab:exp_setup}. We implement the proposed model using Pytorch \cite{paszke2019pytorch}. All experiments are conducted on GeForce RTX 3090 GPUs.

\begin{table*}[]
	\centering
	\caption{Experimental Setup on Datasets.}
	\label{tab:exp_setup}
	\resizebox{7in}{!}{%
		\begin{tabular}{cccccccccccccccccccc}
			\hline
			Dataset & $N_{train}$ & $N_{val}$ & $N_{test}$ & $(C, H, W)$ & $T_{in}$ & $T_{out}$ & $N_S$ & $C_S$ & $N_h$ & $C_h$ & $N_T$ & $C_T$ & $Batch$ & LR & Epoch&$\lambda_0$&$\lambda_1$&$\lambda_2$&$\lambda_3$ \\ \hline
			TaxiBJ       & 19560 & -    & 1334 & (2, 32, 32)   & 4  & 4  & 3 & 64  & 1 & 256 & 1 & 256 & 8  & 0.01 & 50  & 1 & 0.1 & 1 & 1\\
			Moving MNIST & 10000 & 3000 & 10000 & (1, 64, 64)   & 10 & 10 & 4 & 64  & 2 & 512 & 2 & 512 & 16 & 0.01 & 2000& 1 & 0   & 0 & 0 \\
			CIKM         & 8000  & 2000 & 4000 & (1, 128, 128) & 5  & 10 & 2 & 8   & 1 & 32  & 1 & 32  &  2 & 0.01 & 30  & 1 & 1   & 0 & 0  \\ \hline
		\end{tabular}%
	}
\end{table*}

\subsection{One-step-ahead Traffic Flow Forecasting}
We first test our model's ability for short-term forecasting using the TaxiBJ dataset, which records the spatiotemporal trajectory data of the taxicab GPS in Beijing. We preprocess and split the data following the settings in \cite{zhang2017deep}. Each sample contains 8 consecutive frames of shape $2 \times 32 \times 32$. The two channels in the first dimension represent the traffic flow intensities of entering and leaving the same area. Following \cite{wang2019memory}, we normalize the data into [0, 1] and take 4 frames as input to predict the future 4 frames. The hybrid strategy uses a group size of $\delta=4$.

To evaluate our models quantitatively, we use the mean square error (MSE), the mean absolute error (MAE), andper-frame structural similarity index measure (SSIM) \cite{wang2004image} as evaluation metrics. A lower MSE, MAE, or higher SSIM indicates a better prediction result. We take nine existing models as the baseline for comparison. We directly use the results from original or published works to avoid bias. Specifically, the frame-wise MSEs are mainly referenced from \cite{wang2019memory}, and other metrics are reused from \cite{gao2022simvp}.

As shown in Table \ref{table_TaxiBJ}, our spatiotemporal observer, abbreviated as ST-Observer, makes a successful prediction on TaxiBJ. We mark the best results in boldface and the second with an underline. The results show that the ST-Observer almost outperforms all baselines regarding MSE, MAE, and SSIM. 

We choose the entering traffic flow data for visualization, seeing Fig. \ref{figure_TaxiBJ}. |GT-PF| denotes the absolute errors between the ground truth and predicted frames. The small prediction error indicates the high accuracy of the ST-Observer. Therefore, the results show that our model achieves accurate one-step-ahead prediction in traffic flow prediction.

\begin{figure}[!t]
	\centering
	\includegraphics[width=3.5in]{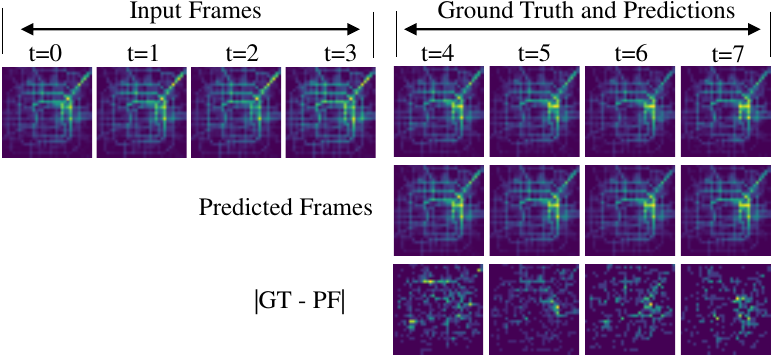}
	\caption{Visualization of one-step-ahead traffic flow forecasting on TaxiBJ.}
	\label{figure_TaxiBJ}
\end{figure}
\begin{table}[]
	\centering
	\caption{Performances of different methods on TaxiBJ}
	\label{table_TaxiBJ}
	\resizebox{\columnwidth}{!}{%
		\begin{tabular}{cccccccc}
			\hline
			\multirow{2}{*}{Model}&   \multicolumn{4}{c}{MSE}  &  & \multirow{2}{*}{MAE} & \multirow{2}{*}{SSIM} \\
											& Frame 1  & Frame 2 & Frame 3& Frame 4  & Avg.  &    &   \\ \hline
			ConvLSTM \cite{shi2015convolutional} &  -   & -    & -    & -   & 0.485    & 17.7    & 0.978     \\
			ST-ResNet \cite{zhang2017deep}     & 0.460  & 0.571  & 0.670  & 0.762    & 0.616    & -   & -   \\
			VPN \cite{kalchbrenner2017video}   &  0.427 & 0.548 & 0.645 & 0.721   & 0.585  & -   & -  \\
			FRNN \cite{oliu2018folded}         &  0.331 & 0.461  & 0.518 & 0.619  & 0.482 & -  & -  \\
			PredRNN \cite{wang2017predrnn}     & 0.318 & 0.427  & 0.516   & 0.595  & 0.464 & 17.1 & 0.971    \\
			PredRNN++ \cite{wang2018predrnn++} & 0.319 & 0.399    & 0.500  & 0.573  & 0.448  & 16.9 & 0.977    \\
			E3D-LSTM \cite{wang2018eidetic}    & -  & -   & -  & -   & 0.432   & 16.9  & 0.979     \\
			MIM \cite{wang2019memory}          & \underline{0.309}  & \underline{0.390} & \underline{0.475}  & \underline{0.542}   & 0.429  & 16.6 & 0.971      \\ 
			PhyDNet	\cite{guen2020disentangling}						& - & - & -  & -  & 0.419 & \underline{16.2}& \underline{0.982}  \\
			SimVP \cite{gao2022simvp}	&-	& - & -  & -  & \underline{0.414} & \underline{16.2}& \underline{0.982}  \\ \hline
			\textbf{ST-Observer}  			&\textbf{0.296}&\textbf{0.376} & \textbf{0.449 }&\textbf{0.501}& \textbf{0.406}&\textbf{15.7}&\textbf{0.983}\\ \hline
		\end{tabular}%
	}
\end{table}
\subsection{Long-term Moving MNIST Prediction}
To evaluate the performance on long-term prediction, we apply our model to predict the future 10 frames by taking the previous 10 frames as inputs in the synthetic Moving MNIST dataset. The group size is set as $\delta=10$. We follow the method described in \cite{srivastava2015unsupervised} to generate 10000 sequences from static MNIST dataset \cite{lecun1998gradient} for training. To avoid bias, we use an open dataset in validation and test phases \cite{srivastava2015unsupervised}. The validation set contains 3000 sequences, and the test set has 10000 sequences. Each sequence consists of 20 frames showing two digits moving in a 64 x 64 box. 

Table \ref{table_MovingMNIST} gives the results of different models in terms of MSE, MAE, and SSIM for long-term prediction on the Moving MNIST dataset. Our ST-Luenberger outperforms all baseline models in terms of SSIM, MAE, and MSE. 
We also compare the model complexity among baselines and our model in terms of GPU memory (per sample) and FLOPs (per frame), which are reused from \cite{gao2022simvp, yu2019efficient}. Low memory consumption and FLOPs of the ST-Observer indicate that it requires small computational intensity and resources and can be implicated efficiently.

Furthermore, we visualize and compare the predicted frames of ConvLSTM, PhyDnet, reused from \cite{tan2023openstl}, and our method. Long-term prediction on Moving MNIST is challenging because there exist occlusions between the moving trajectories of two digits. When the moving digits overlap with each other, an information bottleneck occurs. A representative example in Fig. \ref{figure_MNIST} shows that our ST-Observer predicts the exact moving path of digits with well-preserved shapes. Only a subtle prediction error is observed between the predicted frames and the ground truth. Fig \ref{figure_frame_mse} shows quantitatively that our method has a smaller framewise MSE than ConvLSTM and PhyDnet. Therefore, the results show that our model can capture the complex dynamics of moving objects and make long-term forecasting with high accuracy.

\begin{table}[]
	\centering
	\caption{Complexity and performance comparison on Moving MNIST.}
	\label{table_MovingMNIST}
	\resizebox{3.0in}{!}{%
		\begin{tabular}{cccccc}
			\hline
			Models        & Memory (MB)	& FLOPs	(G)		& SSIM				& MAE   			& MSE   \\ \hline
			ConvLSTM     	&	1043 	& 107.4		&  0.707 			& 182.9 			& 103.3 \\
			TrajGRU \cite{shi2017deep}&-	&-				& 0.713 			& 190.1 			& 106.9 \\
			DFN \cite{jia2016dynamic}&	-	& - 			& 0.726				& 172.8 			& 89.0  \\
			FRNN& 717 &	80.1		& 0.813 			& 150.3 			& 69.7  \\
			VPN    			& 5206 	& 	309.6		& 0.870 & 131.0 			& 64.1  \\
			PredRNN     	& 1666 	&   192.9 	& 0.867 			& 126.1 & 56.8  \\
			CausalLSTM    	&  2017 	&   106.8		& 0.898 			& - 				& 46.5  \\
			MIM				& -	&  115.9 	& 0.910 			& 101.1				 & 44.2 \\
			E3D-LSTM    	& 2695 	&   381.3		& 0.920 			& - 				& 41.3  \\
			CrevNet \cite{yu2019efficient}& \underline{224} &	\underline{1.652}	& 0.947 			& - 				& 24.4  \\
			PhyDNet 		&	 \textbf{200}	& \textbf{1.633}	    & 0.947 			& 70.3 				& 24.4  \\
			SimVP  			&	412	& 1.676	    & \underline{0.948} 			& \underline{68.9}				& \underline{23.8}  \\ \hline
			\textbf{ST-Observer} 	&266&   2.104		& \textbf{0.954} 	& \textbf{63.9} 	& \textbf{21.2 } \\ \hline
		\end{tabular}%
	}
\end{table}
\begin{figure}[!t]
	\centering
	\includegraphics[width=3.4in]{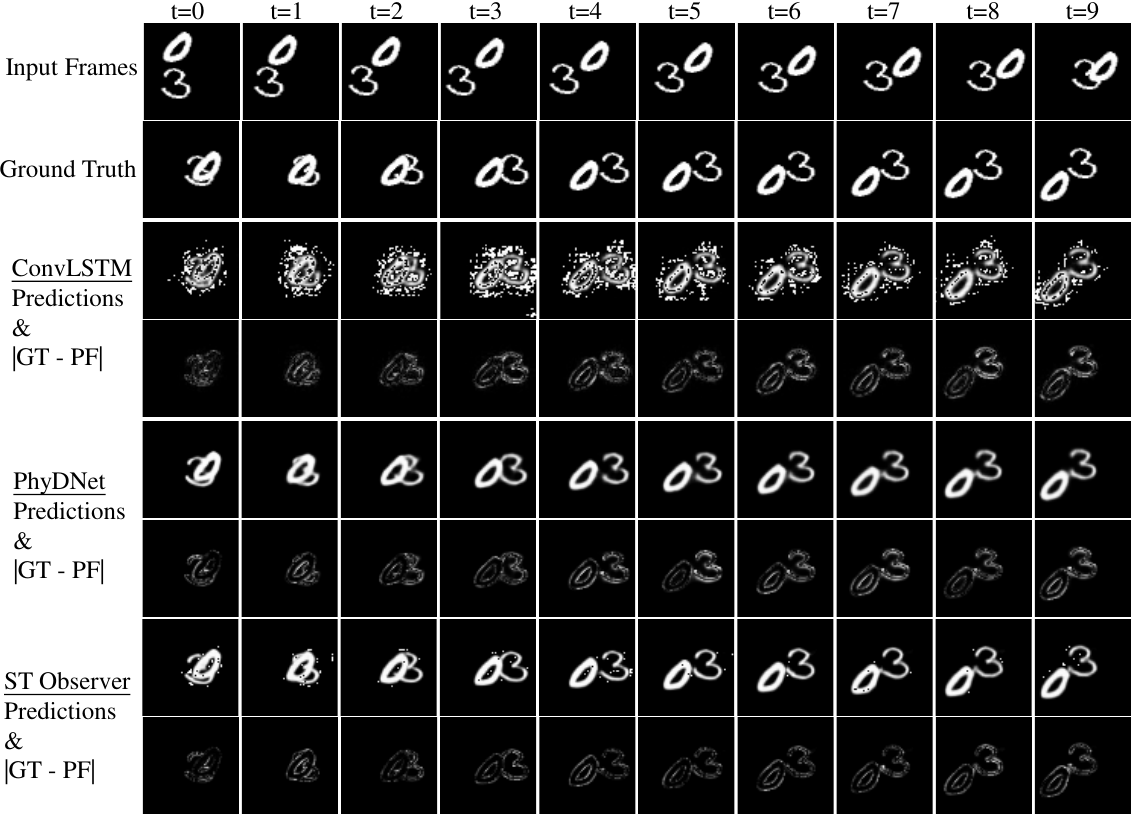}
	\caption{Visualization of prediction examples on Moving MNIST.}
	\label{figure_MNIST}
\end{figure}
\begin{figure}[!t]
	\centering
	\includegraphics[width=2.9in]{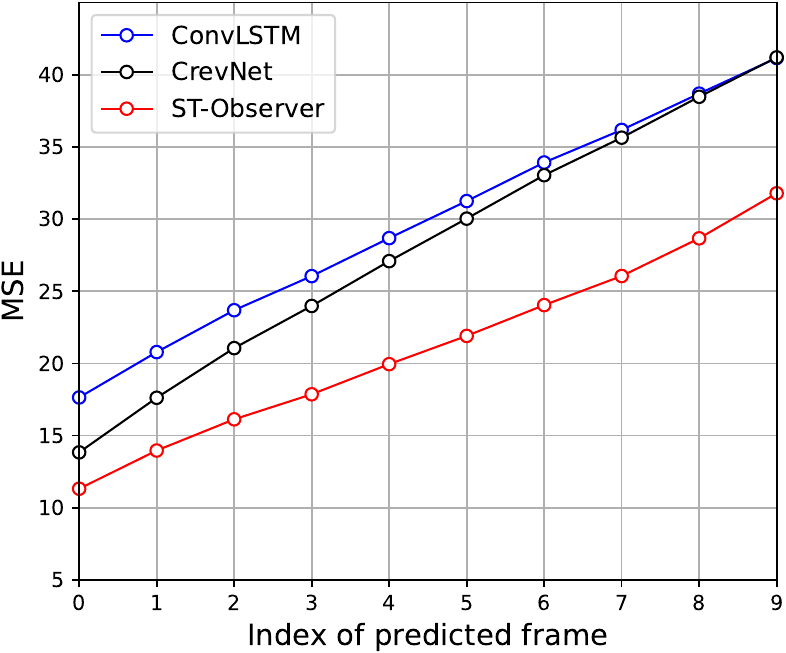}
	\caption{Framewise MSE comparison of baselines and ST-Observer.}
	\label{figure_frame_mse}
\end{figure}

\subsection{Multi-step-ahead Precipitation Nowcasting}
The CIKM dataset is an open dataset for precipitation forecasting, in which the radar echo maps cover the $101 \times 101 km^2$  in Shenzhen, China. Each pixel represents the average value of radar reflectivity in a square area of $1 \times 1 km^2$. Following the setting in \cite{zhang2022rap}, we preprocess and split the original dataset. Finally, the training set has 8000 sequences, the validation set contains 2000 sequences, and the test set has 4000 sequences. Each sequence records 15 consecutive snapshots within 90 minutes.

In this case, we aim to evaluate the model's ability to predict the future with flexible length. We use 5 frames as input to predict the future 10 frames. The hybrid strategy adopts a group size of $\delta=5$. Therefore, the ST-Observer would work recursively like RNN. During the forecasting stage, the model takes 5 frames as inputs to predict 5 future frames, then takes the predicted frames as inputs to predict the subsequent 5 frames.

We first transform the pixel value $p$ into the radar reflectivity by Eq. \eqref{dBZ} to evaluate the model's predictive performance.
\begin{equation}\label{dBZ}
	dBZ=p\times \frac{95}{255}-10
\end{equation}
Then, three commonly used metrics in precipitation nowcasting are employed to evaluate the results. They are MAE, Heidk Skill Score (HSS), and Critical Success Index (CSI). MAE measures the overall error of model prediction, while HSS and CSI measure the accuracy value after exceeding a certain threshold, that is, paying more attention to the error of extreme values. The binary results of predictions and ground truth are calculated by comparing their radar reflectivity with a threshold. We set the thresholds of radar reflectivity as 5, 20, and 40 $dBZ$. After that, by counting the binary results, we can obtain a list of the true positive (TP, prediction=1, truth=1), false negative (FN, prediction=0, truth=1), false positive (FP, prediction=1, truth=0), and true negative (TN, prediction=0, truth=0). Finally, we obtain the HSS and CSI using the following equations.
\begin{equation}\label{HSS}
	HSS = \frac{2(TP \times TN-FN\times FP)}{(TP+FN)(FN+TN)+(TP+FP)(FP+TN)}
\end{equation}
\begin{equation}\label{CSI}
	CSI = \frac{TP}{TP + FN + FP}
\end{equation}

Table \ref{Results_CIKM} gives the results on this dataset. Our model achieves the smallest MAE and the highest average CSI compared with all other models. Besides, it also shows that our model performs superior for nowcasting with a high threshold than other baseline models. For example, ConvLSTM has the best scores in CSI with the threshold of 5 and 20, but its results on the high threshold of 40 are worse than our method. For low thresholds, the ST-Observer makes the second-best results in HSS with thresholds of 5/20, and achieve best in CSI with thresholds of 20 and second-best in CSI with thresholds of 5.

We visualize the results on this dataset in Fig. \ref{Result_cikm}. The radar reflectivity values can refer to the color bar at the bottom of Fig. \ref{Result_cikm}. The model predicts a smooth result, and the sharp areas are filtered, which is the reason why our model does not achieve best in HSS and CSI of high thresholds. The difference map shows that the errors are mostly below 20 $dBZ$, and only a tiny part equals or exceeds 30 $dBZ$. Finally, we can conclude that the proposed method can make multi-step-ahead predictions and achieve high accuracy.

\begin{figure*}[!t]
	\centering
	\includegraphics[width=6in]{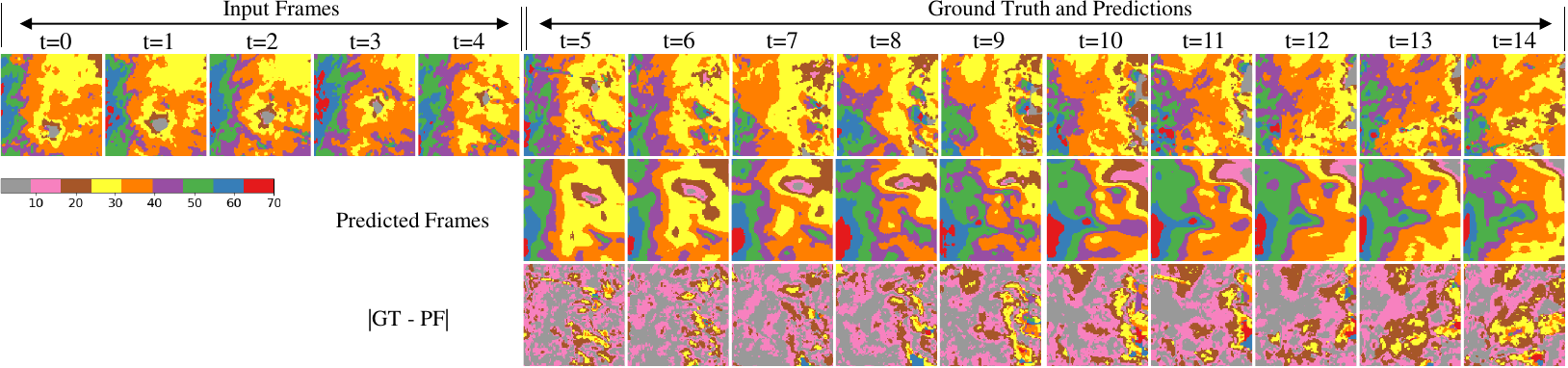}
	\caption{Visualization of prediction examples on CIKM.}
	\label{Result_cikm}
\end{figure*}

\begin{table}[]
	\centering
	\caption{Comparison Results on CIKM in terms of HSS, CSI, and MAE.}
	\label{Results_CIKM}
	\resizebox{\columnwidth}{!}{%
		\begin{tabular}{cccccccccc}
			\hline
			\multirow{2}{*}{Models} & \multicolumn{4}{c}{HSS}           & \multicolumn{4}{c}{CSI}           & \multirow{2}{*}{MAE} \\
			& 5      & 20     & 40     & avg.   & 5      & 20     & 40     & avg.   &                      \\ \hline
			ConvLSTM                & \textbf{0.7031} & 0.4857 & 0.1470 & \underline{0.4453} & \textbf{0.7663} & \underline{0.4092} & 0.0801 & \underline{0.4186} & 5.97                 \\
			ConvGRU                 & 0.6816 & 0.4827 & 0.1225 & 0.4289 & 0.7522 & 0.3952 & 0.0657 & 0.4043 & 6.00                 \\
			TrajGRU                 & 0.6809 & \textbf{0.4945} & \textbf{0.1907} & \textbf{0.4553} & 0.7466 & 0.4028 & \textbf{0.1061} & 0.4185 & \underline{5.90}                 \\
			DFN                     & 0.6772 & 0.4719 & 0.1306 & 0.4266 & 0.7489 & 0.3771 & 0.0704 & 0.3988 & 6.03                 \\
			PhyDNet                 & 0.6741 & 0.4709 & \underline{0.1832} & 0.4427 & 0.7402 & 0.4003 & \underline{0.1017} & 0.4141 & 6.25                 \\
			CMS-LSTM                & 0.6835 & 0.4605 & 0.1720 & 0.4387 & 0.7567 & 0.3788 & 0.0948 & 0.4101 & 5.95                 \\ \hline
			\textbf{ST-Observer}     & \underline{0.6880} & \underline{0.4846} & 0.1588 & 0.4438 & \underline{0.7627} & \textbf{0.4122} & 0.0979 & \textbf{0.4243} & \textbf{5.66}                               \\ \hline
		\end{tabular}%
	}
\end{table}

\subsection{Ablation Study}
\subsubsection{Different Configuration of $A$ and $B$}
In Section \ref{function_app}, we mentioned that $A$ and $B$ in the prediction equation \eqref{eq_pre} have multiple parameterization methods. Therefore, we conducted various ablation tests on the TaxiBJ dataset.

First, the sigmoid and clamp functions in Pytorch were used to limit the value range of elements in $A$. At the same time, an experiment with no restrictions on $A$ was conducted as a comparison, abbreviated as 'None' in table \ref{tab:ablation_AB}. Secondly, regarding the initialization of $A$, we used three methods, namely normal, uniform, and Kaiming uniform. As shown in table \ref{tab:ablation_AB}, parameter $A$, constrained with Sigmoind and initialized by Kaiming uniform, performs better and can be used as the default selection.

Regarding parameterization of $B$, we conducted comparative experiments without using $B$, abbreviated as 'None', and three other experiments using 1*1 convolution, 3*3 convolution, and inception module, respectively. The results show that parametering $B$ with Inception performs best.
\begin{table*}[]
	\centering
	\caption{Results of ablation study on $A$ and $B$}
	\label{tab:ablation_AB}
	\resizebox{5.6in}{!}{%
		\begin{tabular}{ccccccccccc}
			\hline
			& \multicolumn{6}{c}{A}                                                                    & \multicolumn{4}{c}{B}                                 \\
			& \multicolumn{3}{c}{Constrains}            & \multicolumn{3}{c}{Initialization}           & \multicolumn{4}{c}{Parameterization}                  \\
			& Sigmoind                  & Clamp & None  & normal & uniform & KM-uniform                & None  & 1*1Conv & 3*3Conv & Inception                 \\ \hline
			MSE  & 0.423                     & 0.432 & 0.428 & 0.421  & 0.415   & 0.406                     & 0.419 & 0.411   & 0.423   & 0.410                     \\
			MAE  & 15.8                      & 15.9  & 15.9  & 15.9   & 15.9    & 15.7                      & 15.8  & 15.8    & 15.8    & 15.8                      \\
			SSIM & 0.983                     & 0.983 & 0.982 & 0.982  & 0.923   & 0.983                     & 0.983 & 0.983   & 0.983   & 0.982                     \\ \hline
			& $\checkmark$ &       &       &        &         & $\checkmark$ &       &         &         & $\checkmark$ \\ \hline
		\end{tabular}%
	}
\end{table*}
\subsubsection{Effect of Learning with Dynamic Regularization}
The dynamic regularization in the loss \eqref{loss_all} is inferred based on the observer theory. How it would make influence the model's performance need to be investigated. In this section, we explore and discuss the ablation studies on the model's performance with different configurations of dynamic regularization on TaxiBJ and CIKM datasets.

Table \ref{tab:ab_dyn} lists the results of ablation experiments. We first compared three different weights of MAE term in the loss function, which are 0, 0.1, and 1. When $\lambda_1=1$ the model achieves the best performance on both databases. Therefore, $\lambda_1=1$ is set as the default value in subsequent experiments. Then we set one of $\lambda_2, \lambda_3, \lambda_4$ to 0.1 or 1, and the other values to 0 for experiments. It can be found that any different value in $\lambda_2, \lambda_3, \lambda_4$ will affect the accuracy of the model. Finally, we set $\lambda_2, \lambda_3, \lambda_4$ to the same values and the optimal values of the previous experiments for experiments. The combination of their different coefficients will also have a greater impact on the accuracy of the model. For example, when we use the combination of $\lambda_2=0.1, \lambda_3=1, \lambda_4=1$, the model achieves the best performance on TaxiBJ. To sum up, weights of dynamical regularization have an impact on model accuracy. Appropriate selection of weights of dynamical regularization can effectively improve the performance of the model.
\begin{table}[]
	\centering
	\caption{Results of ablation study on dynamic regularization}
	\label{tab:ab_dyn}
	\resizebox{\columnwidth}{!}{%
		\begin{tabular}{cccccccccc}
			\hline
			\multirow{2}{*}{$\lambda_1$} & \multirow{2}{*}{$\lambda_2$} & \multirow{2}{*}{$\lambda_3$} & \multirow{2}{*}{$\lambda_4$} &  & TaxiBJ &  &  & CIKM &  \\
			&     &     &     & MSE             & MAE              & SSIM            & MAE           & HSS             & CSI             \\ \hline
			0   & 0   & 0   & 0   & 0.4297          & 16.4698          & 0.9819          & 6.10          & 0.4277          & 0.4128          \\
			0.1 & 0   & 0   & 0   & 0.4217          & 15.9855          & 0.9825          & 5.78          & 0.4277          & 0.4125          \\
			1   & 0   & 0   & 0   & 0.4085          & 15.7766          & 0.9826          & 5.65          & 0.4426          & 0.4236          \\ \hline
			1   & 0.1 & 0   & 0   & 0.4072          & 15.7208          & 0.9826          & 5.73          & 0.4332          & 0.4174          \\
			1   & 1   & 0   & 0   & 0.4231          & 15.9680          & 0.9825          & \textbf{5.66} & \textbf{0.4438} & \textbf{0.4243} \\
			1   & 0   & 0.1 & 0   & 0.4132          & 15.7796          & 0.9825          & 5.65          & 0.4367          & 0.4198          \\
			1   & 0   & 1   & 0   & 0.4048          & 15.7097          & 0.9828          & 5.65          & 0.4376          & 0.4201          \\
			1   & 0   & 0   & 0.1 & 0.4131          & 15.7747          & 0.9825          & 5.67          & 0.4367          & 0.4207          \\
			1   & 0   & 0   & 1   & 0.4096          & 15.6938          & 0.9828          & 5.60          & 0.4387          & 0.4204          \\ \hline
			1   & 0.1 & 0.1 & 0.1 & 0.4051          & 15.7653          & 0.9827          & 5.69          & 0.4405          & 0.4224          \\
			1   & 1   & 1   & 1   & 0.4230          & 15.8517          & 0.9824          & 5.69          & 0.4428          & 0.4240          \\
			1   & 0.1 & 1   & 1   & \textbf{0.4055} & \textbf{15.7496} & \textbf{0.9828} & 5.64          & 0.4426          & 0.4237          \\
			1   & 1   & 0   & 1   & 0.4309          & 15.8594          & 0.9826         & 5.71          & 0.4279          & 0.4136          \\ \hline
		\end{tabular}%
	}
\end{table}

\section{Conclusion}
\label{sec: Conclusion}
A spatiotemporal observer is designed for predictive learning of high-dimensional data based on the traditional Kazantzis-Kravaris-Luenberger observer of low-dimensional systems. The proposed spatiotemporal observer has convergence guarantees and upper-bounded generalization errors. It could be integrated with existing neural network modules and serve as a powerful architecture. In this work, we instantiate this framework with CNNs for modeling and forecasting high-dimensional data. Extensive experiments validate the effectiveness of the proposed method. We hope this work could give a new view to designing the architecture of neural networks for modeling and forecasting spatiotemporal data.

\bibliography{SpatioTemporalObserver.bib}
\bibliographystyle{IEEEtran}

\begin{IEEEbiography}[{\includegraphics[width=1in,height=1.25in,clip,keepaspectratio]{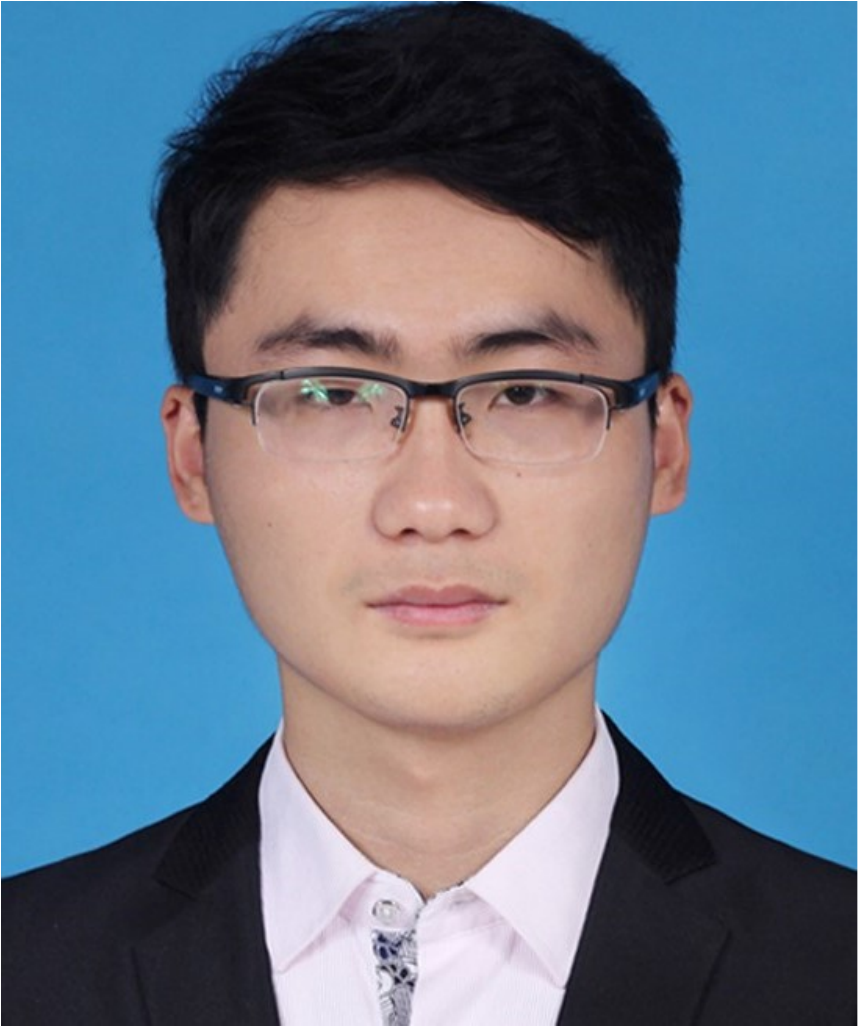}}]{Tongyi Liang}
	received the B.E. degree in automotive engineering from the University of Science and Technology Beijing, Beijing, China, in 2017, the M.E. degree in automotive engineering from the Beihang University, Beijing, China, in 2020. He is currently working toward the Ph.D degree with the Department of Systems Engineering, City University of Hong Kong, Hong Kong, China. 
	
	His current research interests focus on neural networks and deep learning.
\end{IEEEbiography}
\begin{IEEEbiography}[{\includegraphics[width=1in,height=1.25in,clip,keepaspectratio]{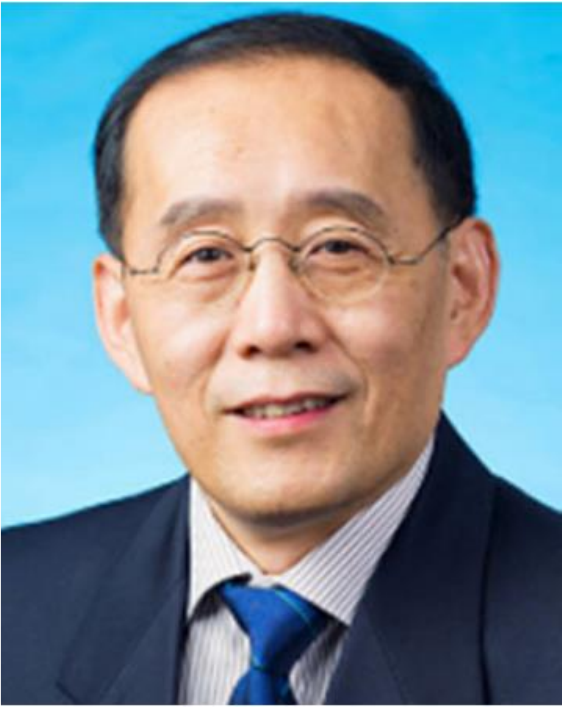}}]{Han-Xiong Li (Fellow, IEEE)}
received the B.E. degree in aerospace engineering from the National University of Defense Technology, Changsha, China, in 1982, the M.E. degree in electrical engineering from the Delft University of Technology, Delft, The Netherlands, in 1991, and the Ph.D. degree in electrical engineering from the University of Auckland, Auckland, New Zealand, in 1997.

He is the Chair Professor with the Department of Systems Engineering, City University of Hong Kong, Hong Kong. He has a broad experience in both academia and industry. He has authored two books and about 20 patents, and authored or coauthored more than 250 SCI journal papers with h-index 52 (web of science). His current research interests include process modeling and control, distributed parameter systems, and system intelligence.

Dr. Li is currently the Associate Editor for IEEE Transactions on SMC: System and was an Associate Editor for IEEE Transactions on Cybernetics (2002-–2016) and  IEEE Transactions on Industrial Electronics (2009–-2015). He was the recipient of the Distinguished Young Scholar (overseas) by the China National Science Foundation in 2004, Chang Jiang Professorship by the Ministry of Education, China in 2006, and National Professorship with China Thousand Talents Program in 2010. Since 2014, he has been rated as a highly cited scholar in China by Elsevier.
\end{IEEEbiography}

\end{document}